%% file: main.tex
\documentclass[sigconf,screen,natbib]{acmart}
\usepackage{multirow}
\usepackage{graphicx}
\usepackage{subcaption}
\usepackage{float}
\usepackage[most]{tcolorbox}
\usepackage{algorithm}
\usepackage{algorithmicx}
\usepackage{algpseudocode}
\usepackage{enumitem}
\usepackage{arydshln} 
\usepackage{microtype}
\usepackage{url}
\usepackage{stfloats}
\usepackage{tablefootnote}
\usepackage{appendix}
\usepackage{amsmath}
\usepackage{amsthm}

\newcommand{\ignore}[1]{}

\newcommand{\ie}{\emph{i.e., }}

\newcommand{\eg}{\emph{e.g., }}

\DeclareMathOperator*{\argmax}{arg\,max}

\newtheoremstyle{lemma_nonumber}
  {}{}{\itshape}{}{\bfseries}{.}{ }{}
\theoremstyle{lemma_nonumber}
\newtheorem*{lemma*}{Lemma}

\AtBeginDocument{%
  }

\copyrightyear{2025} 
\acmYear{2025} 
\setcopyright{acmlicensed}\acmConference[CIKM '25]{Proceedings of the 34th
 ACM International Conference on Information and Knowledge
 Management}{November 10--14, 2025}{Seoul, Republic of Korea}
 \acmBooktitle{Proceedings of the 34th ACM International Conference on
 Information and Knowledge Management (CIKM '25), November 10--14, 2025,
 Seoul, Republic of Korea}
 \acmDOI{10.1145/3746252.3761167}
 \acmISBN{979-8-4007-2040-6/2025/11}

\begin{document}

\title{From Anchors to Answers: A Novel Node Tokenizer for Integrating Graph Structure into Large Language Models}  
% Scalable Position-anchored Graph Tokenizer for  Graph Retrieval-Augmented Generation}
% \title{Scalable Position-anchored Graph Tokenizer for  Graph Retrieval-Augmented Generation}

\settopmatter{authorsperrow=4}
\author{Yanbiao Ji}
\authornote{Both authors contributed equally to this research.}
  \affiliation{%
  \institution{Shanghai Jiao Tong University}
  \city{Shanghai}
  \country{China}
}
\email{jiyanbiao@sjtu.edu.cn}

\author{Chang Liu}
\authornotemark[1]
\authornote{This work was done while Chang Liu was an intern at Tencent.}
  \affiliation{%
  \institution{Shanghai Jiao Tong University}
  \city{Shanghai}
  \country{China}
}
\email{isonomialiu@sjtu.edu.cn}

\author{Xin Chen}
  \affiliation{%
  \institution{The Chinese University of Hong Kong}
  \city{Hong Kong}
  \country{China}
}
\email{xchen@se.cuhk.edu.hk}

\author{Dan Luo}
  \affiliation{%
  \institution{Lehigh University}
  \city{Bethlehem}
  \state{PA}
  \country{USA}
}
\email{danluo.ir@gmail.com}

\author{Mei Li}
  \affiliation{%
  \institution{Shanghai Jiao Tong University}
  \city{Shanghai}
  \country{China}
}
\email{mei-li@sjtu.edu.cn}

\author{Yue Ding}
\authornote{Corresponding authors: Yue Ding and Wenqing Lin.}
  \affiliation{%
  \institution{Shanghai Jiao Tong University}
  \city{Shanghai}
  \country{China}
}
\email{dingyue@sjtu.edu.cn}

\author{Wenqing Lin}
\authornotemark[3]
  \affiliation{%
  \institution{Tencent}
  \city{Shenzhen}
  \country{China}
}
\email{edwlin@tencent.com}

\author{Hongtao Lu}
  \affiliation{%
  \institution{Shanghai Jiao Tong University}
  \city{Shanghai}
  \country{China}
}
\email{htlu@sjtu.edu.cn}

\renewcommand{\shortauthors}{Yanbiao Ji et al.}

\input{0-Abstract}

\ccsdesc[500]{Information systems~Data mining}

% \ccsdesc[500]{Information systems~Recommender systems}

\keywords{Large Language Models, Positional Encoding, Knowledge Graphs}



\maketitle

\input{1-Introduction}
\input{2-RelatedWork}
\input{3-Preliminary}

\input{4-Methodology}

\input{5-Experiments}

\vspace{-5pt}
\section{Conlusion}

In the paper, we propose NT-LLM, an anchor-based graph positional encoding approach that enables efficient graph tokenization for LLMs. Our method preserves crucial structural information through anchor nodes selection without requiring extensive textual descriptions or complex GNNs. 
% NT-LLM uses strategically selected anchor nodes and a rank-preserving pretraining objective to efficiently tokenize node in graph data. This strategy effectively bridges non-Euclidean graph topology with LLMs' Euclidean embedding space, providing important structural information for LLMs. 
Evaluations across diverse benchmarks demonstrate significant improvements across diverse tasks from node classification to complex reasoning, confirming the effectiveness and efficiency of our proposed NT-LLM method.

\bibliographystyle{ACM-Reference-Format}
\balance
\bibliography{reference}

% \newpage{}

\end{document}

%% file: 0-Abstract.tex
\begin{abstract}
% graph learning is important
Enabling large language models (LLMs) to effectively process and reason with graph-structured data remains a significant challenge despite their remarkable success in natural language tasks. Current approaches either convert graph structures into verbose textual descriptions, consuming substantial computational resources, or employ complex graph neural networks as tokenizers, which introduce significant training overhead. To bridge this gap, we present NT-LLM, a novel framework with an anchor-based positional encoding scheme for graph representation. Our approach strategically selects reference nodes as anchors and encodes each node's position relative to these anchors, capturing essential topological information without the computational burden of existing methods. Notably, we identify and address a fundamental issue: the inherent misalignment between discrete hop-based distances in graphs and continuous distances in embedding spaces. By implementing a rank-preserving objective for positional encoding pretraining, NT-LLM achieves superior performance across diverse graph tasks ranging from basic structural analysis to complex reasoning scenarios. Our comprehensive evaluation demonstrates that this lightweight yet powerful approach effectively enhances LLMs' ability to understand and reason with graph-structured information, offering an efficient solution for graph-based applications of language models.
\end{abstract}

%% file: 1-Introduction.tex
\section{Introduction}
In recent years, Large Language Models~(LLMs), such as LLaMA~\cite{llama} and GPT~\cite{gpt4}, have revolutionized artificial intelligence.
They have demonstrated powerful capabilities in solving various natural language processing~(NLP) tasks, including question answering~\cite{qa1, qa2}, text generation~\cite{textgen1, textgen2}, and document understanding~\cite{doc1, doc2}.
While LLMs have primarily been applied to text data, an increasing number of applications now involve text data intertwined with structured information represented as graphs. For instance, in social networks, nodes represent entities, while edges capture the relationships between them. Both nodes and edges can also be associated with textual descriptions that detail their attributes.
Since LLMs are primarily designed to model text in a sequential format, applying them to graph-related tasks presents new challenges, particularly in encoding the structural information of graphs~\cite{llmgraph1, llmgraph2}.

While many studies~\cite{graphformer, heterformer, graphfoundation} have attempted to combine language modeling and graph representation learning with medium-sized transformer models such as BERT~\cite{bert} and RoBERTa~\cite{roberta}, efficient graph reasoning with LLMs of billions of parameters remains challenging. To leverage the strength of LLMs for graph structure understanding, existing efforts can be categorized into two groups~\cite{survey2, survey3}: (1) \textbf{Graph Textual Conversion}, which translates a graph's structure into a descriptive textual representation~\cite{graph_adapter, llaga, walklm}. 
These studies typically convert the local context of a target node into textual descriptions that incorporate relevant structural information, and then utilize large language models to predict properties such as node labels and the presence of links. The underlying assumption is that the powerful capabilities of LLMs can generalize to interpret graph-structured knowledge through textual input. 
However, such descriptions typically require a large number of tokens to describe the graph structure, greatly increasing the cost of LLM inference. (2) \textbf{Graph Node Tokenizer}, which generates node embeddings for each node and then projects these embeddings into LLM token space~\cite{graphgpt, higpt, dual}. With the utilization of powerful Graph Neural Networks~(GNNs) as graph node tokenizers, these methods effectively reduce the inference cost by representing the graph structure with compact node tokens. However, the graph representation learning process often brings heavy training overhead. Achieving scalability comparable to LLMs requires an expressive GNN (\eg with elaborate graph convolution paradigms) of similar scale, which introduces additional computational overhead.

To enable effective and efficient LLM reasoning on graphs, a graph encoding paradigm that preserves rich graph structural information without introducing heavy training or inference overhead is needed. This naturally aligns with the motivation of graph positional encoding, which introduces extra embeddings containing structural information to disambiguate nodes and enhance graph representation learning during the training of GNNs and graph transformers~\cite{lap1, dist1, rand1}. In this paper, we introduce an anchor-based graph positional encoding scheme for graph node tokenization, and investigate its integration with LLMs across various graph-related tasks. The core of our method is the strategic selection of key nodes, referred to as anchors, which serve as reference points for encoding the graph topology.
Each node is then represented based on its relative distance to these anchors, effectively capturing the structural information of the graph. Furthermore, we identify the issue of misalignment between the non-Euclidean graph space~(hop-based discrete distance) and the Euclidean embedding space~(continuous Euclidean distance). A rank-preserving pretraining objective is proposed to project the positional embedding into Euclidean space. We then apply task-specific tuning procedures using prompt tuning and LoRA techniques to facilitate better structural understanding of LLMs for downstream tasks. Extensive empirical studies demonstrate that NT-LLM substantially improves LLM performance across a diverse range of graph-related tasks, from basic graph analysis to complex reasoning.
Our main contributions are as follows:
\begin{itemize}[leftmargin=*]
\item We introduce a position-anchored graph encoding approach for LLMs that efficiently preserves crucial structural information while reducing the computational complexity associated with commonly used graph encoding methods.

\item We identify and address the issue of misalignment between the non-Euclidean graph space and the Euclidean embedding space, which hinders the effectiveness of graph positional embedding in graph reasoning with LLMs. 

\item We conduct an extensive empirical evaluation on multiple graph benchmarks, covering a wide range of task complexities and graph types. Our results provide insights into the performance and generalizability of NT-LLM, highlighting its potential for adoption in various graph learning scenarios.
\end{itemize}

%% file: 2-RelatedWork.tex
\begin{table*}[t]
\footnotesize
    \caption{Comparative analysis of graph positional encoding techniques, including our proposed method.}
    % \vspace{-10pt}
    \centering
    \resizebox*{\textwidth}{!}{
    \begin{tabular}{c|cccccc}
    \toprule
         & \textbf{Laplacian Eigenmap}~\cite{lap1}& \textbf{DeepWalk}~\cite{walk1}& \textbf{PGNN}~\cite{pgnn2} & \textbf{HPLC}~\cite{dist2} & \textbf{RFP}~\cite{rand1}&\textbf{Ours} \\
         % &\textbf{Encoding Scheme}&\textbf{Local Structure}&\textbf{Global Position}&\textbf{Euclidean Embedding}&\textbf{Perturbation Robust}&\textbf{Dynamic Graph}&\textbf{Time Complexity}
    \midrule
    \textbf{Encoding Scheme}&eigenvectors&random walk&distance&distance, eigenvectors&random feature&distance\\
    \textbf{Local Structure}&$\checkmark$&$\checkmark$&$\checkmark$&$\checkmark$&$\checkmark$&$\checkmark$\\
    \textbf{Global Position}&$\times$&$\times$&$\checkmark$&$\checkmark$&$\times$&$\checkmark$\\
    \textbf{Euclidean Space}&$\checkmark$&$\checkmark$&$\times$&$\checkmark$&$\checkmark$&$\checkmark$\\
    \textbf{Time Complexity}&$O(|\mathcal{V}|^3)$&$O(|\mathcal{E}|)$&$O(|\mathcal{V}|^2log^2(|\mathcal{V}|))$&$O(|\mathcal{E}|log(|\mathcal{V}|)+|\mathcal{V}|log^2(|\mathcal{V}|))$&-&$O(|\mathcal{V}|^2 + |\mathcal{V}| |\mathcal{E}|)$\\
    \bottomrule
    \end{tabular}}
    \label{tab:pe}
    \vspace{-5pt}
\end{table*}
\section{Related Work}
\subsection{Graph Positional Encoding}

Graph Neural Networks (GNNs) have significantly advanced graph representation learning by enabling the extraction of meaningful embeddings from graph-structured data through message-passing mechanisms~\cite{gnn1, gnn2, gnn3, gnn4, gnn5}. However, standard GNN architectures often struggle to differentiate among nodes with similar local structures but different positions within the global graph topology. Graph positional encoding addresses this limitation by enhancing node representations with positional information, allowing the capture of important structural features.

Several approaches have been developed to encode positional information in graphs. Laplacian eigenmaps~\cite{lap1, lap2} utilize the eigenvectors of the graph Laplacian matrix for this purpose. In contrast, random walk encodings~\cite{walk1, walk2, walk3} capture structural information by simulating random walks on the graph. This method encodes the co-occurrence probabilities of nodes during these walks, thereby embedding nodes with similar neighborhoods closer in the embedding space.
Rx`ecently, researchers have introduced Distance Encoding~\cite{dist1, dist3, dist4}, which incorporates structural information by encoding the shortest path distances between nodes. Furthermore, Random Feature methods~\cite{rand1, rand2} have been developed to approximate positional encodings using learnable or predefined random feature maps. 
To provide a comprehensive overview of these approaches, Table~\ref{tab:pe} presents a detailed comparison of various graph positional encoding methods.

\subsection{LLMs in Graph-Related Tasks}

The rapid advancement in LLMs have led to their successful application across various domains, leveraging their powerful sequence modeling capabilities~\cite{llm1, llm2, llm3}. In recent years, there has been a growing interest in applying LLMs to graph-related tasks, aiming to harness their ability to capture long-range dependencies and perform complex reasoning. 

Initial efforts focused on directly feeding textual descriptions of graphs into LLMs to tackle tasks such as node classification and link prediction~\cite{llm_only_1, llm_only_2}. While these methods demonstrated the potential of LLMs in understanding graph data, they faced significant scalability challenges due to the complexity of constructing comprehensive prompts and the loss of crucial structural information during the graph-to-text conversion process.
To address these limitations, subsequent research has explored the integration of Graph Neural Networks (GNNs) with LLMs to better leverage the strengths of both paradigms~\cite{graphgpt, grag, gretri}. One common approach involves using GNNs to generate structure-aware embeddings, which are then fed into LLMs for downstream tasks~\cite{graphgpt, higpt}. More advanced techniques have delved into model fusion training~\cite{engine}, model alignment~\cite{g2p2, grenade}, and the development of LLM agents specifically designed to handle graph data~\cite{graph-agent-1, graph-agent-2}.

% Despite advancements in integrating LLMs with graph-related tasks, current approaches struggle to enable LLMs to effectively understand and utilize graph structures. Reliance on intermediate processing through GNNs or graph-to-text conversion can result in information loss and hinder LLMs' ability to capture long-range dependencies and global structural patterns. 
% Our research aims to develop a positional encoding technique tailored to LLMs' strengths, unlocking their full potential in graph-related tasks. This approach enables LLMs to capture rich structural information and perform complex reasoning over graph data directly.

%% file: 3-Preliminary.tex
\section{Preliminary}
\paragraph{\textbf{Textual Graphs}}
A textual graph is a graph in which nodes and edges are associated with textual attributes. Formally, it is defined as $\mathcal{G} = (\mathcal{V}, \mathcal{E}, \{\mathbf{T}_v\}_{v\in \mathcal{V}}, \{\mathbf{T}_e\}_{e\in \mathcal{E}})$, where $\mathcal{V}$ and $\mathcal{E}$ represent the sets of nodes and edges, respectively. Here, $\mathbf{T}_v$ and $\mathbf{T}_e$ denote the textual attributes corresponding to each node and edge, which are usually represented by natural language descriptions.\footnote{In this work, we assume that the distance between two adjacent nodes is fixed at 1. The study of weighted graphs, where edge distances may vary, is left for future work.}

% Define a general task how LLM understand graph

\paragraph{\textbf{Text Encoding via Language Models}}
Language Models (LMs) have proven to be highly effective at encoding textual attributes in graphs, producing embeddings that capture rich semantic information. 
For a given textual attribute $T_i$ associated with a node or edge $i$, an LM encodes this attribute into an embedding vector as follows:
\begin{equation}
    \mathbf{x}_i = \textbf{LM}(T_i) \in \mathbb{R}^k.
\end{equation}

% \paragraph{\textbf{Large Language Models (LLMs)}} 
% have revolutionized the field of natural language processing and have demonstrated remarkable success across various domains. 
% These models, 

\paragraph{\textbf{Prompt Tuning for LLMs}}
LLMs are trained on vast corpora of textual data, demonstrating emergent capabilities that facilitate advanced semantic understanding and exceptional task generalization.
Formally, an LLM parameterized by $\theta$ takes as input a sequence of tokens $\mathbf{X} = \{\mathbf{x}_1, \mathbf{x}_2, \ldots, \mathbf{x}_n\}$ along with a task prompt $\mathbf{P}$, and generates an output sequence $\mathbf{Y} = \{\mathbf{y}_1, \mathbf{y}_2, \ldots, \mathbf{y}_r\}$. The probability distribution of the output sequence, conditioned on the concatenated input sequence and prompt $[\mathbf{P};\mathbf{X}]$, is expressed as:
\begin{equation}
    p_{\theta}(\mathbf{Y}|[\mathbf{P};\mathbf{X}]) = \prod_{i=1}^r p_{\theta}(\mathbf{y}_i|\mathbf{y}_{<i},[\mathbf{P};\mathbf{X}]),
\end{equation}
where $\mathbf{y}_{<i}$ represents the prefix of sequence $\mathbf{y}$ up to position $i-1$, and $p_{\theta}(\mathbf{y}_i|\mathbf{y}_{<i},[\mathbf{P};\mathbf{X}])$ denotes the probability of generating token $\mathbf{y}_i$ given the preceding tokens $\mathbf{y}_{<i}$ and the input $[\mathbf{P};\mathbf{X}]$.

Prompt tuning~\cite{ptuning} is an efficient technique for adapting LLMs to specific tasks without modifying the model's parameters. This technique keeps the pretrained LLM frozen, and optimizes a small set of continuous prompt embeddings $\{\mathbf{e}_i\}_{i=1}^{n}$, where $n$ is the number of prompt tokens.
% This approach significantly reduces the number of trainable parameters and storage requirements for task-specific adaptations.
These prompts are generally initialized either randomly or using the embeddings of specific tokens, and are subsequently optimized throughout the training process. Formally, the prompt embeddings can be represented as:
\begin{equation}
    \mathbf{E} = [\mathbf{e}_1, \mathbf{e}_2, ..., \mathbf{e}_n]^T,
\end{equation}
where the dimension of the embedding space is $d$, and $\mathbf{E} \in \mathbb{R}^{n \times d}$. The prompt embeddings can be generated by a small trainable mapping network $\Phi$:
\begin{equation}
    \mathbf{E} = \Phi(\mathbf{X}),
\end{equation}
where $\mathbf{X}$ represents the input embeddings to be transformed. This allows for more flexible and expressive prompt representations. The generation process with prompt tuning can be represented as follows:
\begin{equation}
    p_{\theta,\Phi}(\mathbf{Y}|[\mathbf{P};\mathbf{X}]) = \prod_{i=1}^r p_{\theta,\Phi}(\mathbf{y}_i|\mathbf{y}_{<i},[\mathbf{P};\mathbf{X}]),
\end{equation}
where $\theta$ represents the frozen parameters of the pretrained LLM, $\Phi$ is the learnable prompt mapping network, $\mathbf{P}$ is the prompt, $\mathbf{X}$ is the input sequence, and $\mathbf{Y} = \{\mathbf{y}_1, \mathbf{y}_2, ..., \mathbf{y}_r\}$ is the output sequence.

\begin{figure*}[htbp]
    \centering
    \includegraphics[width=.78\textwidth]{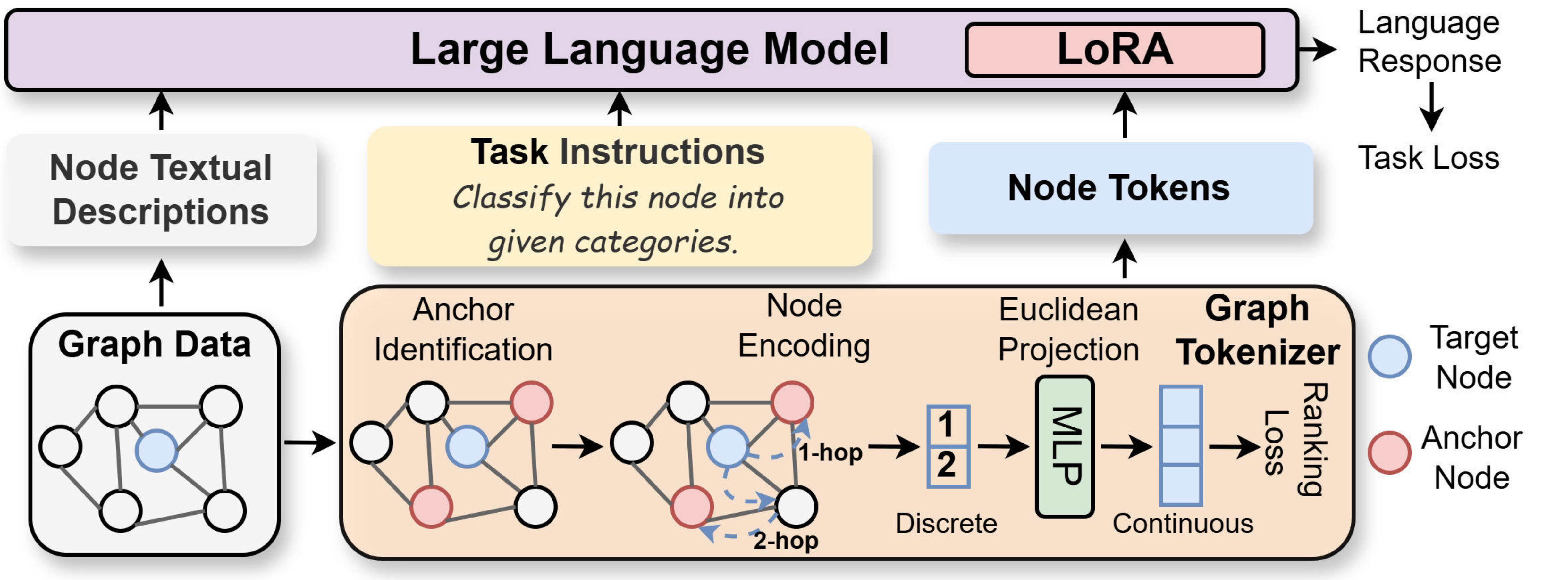}
    % \vspace{-10pt}
    \caption{Overview of our proposed \textit{NT-LLM} approach. 
    It consists of two steps:
    % Our method addresses graph-related tasks through three main steps: 
    (1) Graph Tokenizer: We select key nodes as anchors with a greedy algorithm and compute relative distances between nodes and these anchors to encode the graph structure. The relative distances are then projected into a continuous Euclidean space while preserving the partial ordering of node distances. (2) Task Tuning: We integrate the pretrained embeddings with a large language model using LoRA for task-specific fine-tuning of the LLM, enhancing the performance of downstream graph understanding tasks.}
    \label{fig:pipeline}
    % \vspace{-10pt}
\end{figure*}

%% file: 4-Methodology.tex
\section{Methodology}
We propose NT-LLM, which can seamlessly integrate graph-structure knowledge with LLMs through two key components: \textbf{Graph Node Tokenizer} and \textbf{Task-Specific LLM Tuning}. 
The node tokenizer leverages carefully selected anchor nodes to encode the spatial position of each node, and positional embedding pretraining to preserve geometric relationships between nodes.
The task-specific LLM tuning integrates our node position embedding with prompt tuning and low-rank adaptation, which allows LLMs to effectively leverage both textual and graph-based information.  
% Our proposal first transforms graph nodes into tokens rich in structural information through pretraining. These tokens are then incorporated into the LLM's input. Subsequently, the LLM undergoes task-specific fine-tuning, enabling it to effectively leverage both textual and graph-based information. 
Figure~\ref{fig:pipeline} illustrates the overall framework of NT-LLM.

% \subsection{Node Tokenization}
\subsection{Graph Node Tokenizer}
In large language models, it is straightforward to inject information about the relative or absolute position of tokens in a sequence via their index. 
However, this approach is not feasible for graphs due to two key differences. 
First, \emph{graphs do not have an inherent linear ordering of nodes}, unlike sequences, where tokens follow a clear order. Nodes in a graph are interconnected in a complex, multidimensional structure, where relationships are defined by edges, and there is no natural start or end. 
Second, \emph{the neighborhood of each node can vary significantly in size and shape}, which makes the concept of a relative or absolute ``position'' less meaningful. 
To address this challenge, we propose a novel graph node tokenizer, which consists of three key steps: anchor node identification, node encoding, and Euclidean projection.

% A tokenizer is in charge of preparing the inputs for a model.
% For example, in large language models, it breaks down raw text into smaller units for the model. 
% Similarly, \textcolor{red}{our Node Tokenization transforms rich structural information of graphs into a sequential format assimilable by LLMs}. This process comprises three crucial steps: \textit{Anchor Node Identification}, \textit{Relative Distance Encoding}, and \textit{Positional Embedding Pretraining}.

\subsubsection{Anchor Node Identification}
Prior works~\cite{pgnn2, pgnn1} have demonstrated that using anchor nodes can well capture the position of a given node with respect to all other nodes in a graph. 
In particular, the position of a node can be described in terms of its relative distance (\eg shortest path distance) to these anchor nodes. 
% we first resort to the anchor node. As the anchor nodes serve as reference, the position of other nodes can be described in terms of their relative distance (\ie shortest path length) to these anchor nodes. 
% The shortest path distance can represent the relative position or similarity between two nodes in a graph.
% Ideally, the anchor nodes should be distributed evenly throughout the graph to provide a comprehensive view of the graph structure. In addition, the selection for anchor nodes should be based on the all-pair shortest path distance to ensure optimal coverage and minimize redundancy. However, calculating the shortest path is computationally expensive, since it needs to calculate the all-pair shortest path distance for all node pairs, with the time complexity typically on the order of $O(|\mathcal{V}|^3)$~\cite{apsp}.
For efficient identification of anchor nodes, we implement a greedy anchor selection algorithm with a coverage ratio threshold.
% which can provide an approximation of the all-pair shortest path distance with a theoretical guarantee and thus significantly reduce the computation complexity. 
% To capture the essence of a graph's topology, we propose to encode the nodes in graph by applying an approximate shortest-path distance algorithm. 
% Our algorithm greedily identify a subset of nodes as anchor nodes \textcolor{red}{where during each selection, the node will maximize XXXX.}
% The anchor nodes are denoted as $\mathcal{A} = \{a_1, a_2, \ldots, a_K\}$, where $K$ is a hyperparameter. 
% Inspired by distance estimation techniques~\cite{NgZ02, dist1}, 
The details of this greedy selection procedure are shown in Algorithm~\ref{alg:selection}.
Given a coverage ratio $CR$ and coverage radius $c$, we start with an empty set $\mathcal{A}$ of anchor nodes and an empty set $N_{cover}$ of covered nodes (Line \ref{alg:line-init}). Here, we define that a node $u$ is covered by a node $v$ only if $u$ is in the $c$-hop subgraph of node $v$; otherwise, $u$ is considered uncovered by $v$.
Then, we iteratively select a new anchor node that covers the maximum set of uncovered nodes in its $c$-hop subgraph $N_c(v)$~(Line \ref{alg:line-anchor-select}) and add these covered nodes to $N_{cover}$ (Line \ref{alg:line-cover-node-merge}) until the size of $N_{cover}$ is no less than $CR * |\mathcal{V}|$ (Line \ref{alg:line-cond}).
% maximizes an additional covered nodes. 
% The details of this greedy selection procedure are shown in Algorithm~\ref{alg:selection}.

\begin{algorithm}[htbp]
\caption{Greedy Algorithm for Anchor Node Selection}
\label{alg:selection}
\begin{algorithmic}[1]
\Require Graph $G(\mathcal{V},\mathcal{E})$, target coverage ratio $CR$, coverage radius $c$
\Ensure Set of anchor nodes $\mathcal{A}$
\State Initialize $\mathcal{A} \gets \emptyset$, $N_{cover} \gets \emptyset$ \label{alg:line-init}

\While{$|N_{cover}| < CR*|\mathcal{V}|$} \label{alg:line-cond}

    \State $anchor \gets \argmax_{v \in \mathcal{V} \setminus \mathcal{A}} |N_c(v) \setminus N_{cover}|$ \label{alg:line-anchor-select}
    
    \If{$|N_c(anchor) \setminus N_{cover}| = 0$}
        \State \textbf{break}
    \EndIf
    \State $\mathcal{A} \gets \mathcal{A} \cup \{anchor\}$
    
    \State $N_{cover} \gets N_{cover} \cup N_c(anchor)$ \label{alg:line-cover-node-merge}
\EndWhile
\State \Return $\mathcal{A}$
\end{algorithmic}
\end{algorithm}

% Here, given a coverage ratio $CR$ and coverage radius $c$, $N_c(v)$ denotes the set of nodes in the $c$-hop subgraph of node $v$, and $\mathcal{A} = \{a_1, a_2, \ldots, a_K\}$ denotes the set of anchor nodes, where $K$ is the maximum number of anchor nodes varying with the change of $CR$ and $c$.
% Here, $N_h(v)$ denotes the set of nodes within a predefined maximum hop $c$ from node $v$, and $\mathcal{A} = \{a_1, a_2, \ldots, a_K\}$ denotes the set of anchor nodes, where $K$ is the maximum number of anchor nodes varying with the change of coverage ratio $CR$ and coverage radius $c$. 
% In addition. This greedy approach iteratively selects the node that covers the largest number of previously uncovered nodes until the target coverage ratio $CR$ is achieved.
% The time complexity of this selection process is $O(|\mathcal{V}|^2 \cdot |\mathcal{A}| + |\mathcal{V}|\cdot |\mathcal{E}|)$, where $|\mathcal{A}|$ is typically much smaller than the number of nodes $|\mathcal{V}|$. A comprehensive analysis of the time complexity is provided in Appendix~\ref{appendix:time}.
% Therefore,
% \textcolor{red}{in comparison with existing methods, our method is significantly more efficient.}
% This approach ensures a comprehensive representation of the graph's structure while maintaining computational feasibility.
The identified anchor nodes enable us to provide a unique node description for other nodes in terms of their relative distance, capturing both global and local structures within the graph.
% \paragraph{Relative Distance Encoding}
% Next, we elaborate our approach for relative distance encoding, which approximates the shortest path distance based on selected anchor nodes aforementioned. 
\subsubsection{Node Encoding}
Given the identified anchor nodes $\mathcal{A}=\{a_1, a_2, \ldots, a_K\}$, we encode the position of each node $v$ with respect to these anchors: 
% Specifically, for any node $v$ in the graph, its relative position is represented with respect to these anchors as follows:
\begin{equation}
    \hat{\mathbf{d}}_v = (d_1, d_2, \ldots, d_K),
\end{equation}
\begin{equation}
    d_i = \text{dist}(v, a_i), \quad \forall i \in \{1, \ldots, K\},
\end{equation}
\noindent
where $\text{dist}(v, a_i)$ denotes the number of hops in the shortest path between node $v$ and anchor node $a_i$.

Utilizing relative distance, we can approximate the shortest distance between any two nodes $u$ and $v$ in the graph defined as:
\begin{equation}
    % d(u, v) \approx \hat{d}(u, v) 
    \hat{d}(u, v) := \min_{k \in \{1, \ldots, K\}} \left( \hat{\mathbf{d}}_{u}[k] + \hat{\mathbf{d}}_{v}[k] \right),
\end{equation}
\noindent
where $\hat{\mathbf{d}}_{u}[k]$ means the $k$-th element of $\hat{\mathbf{d}}_{u}$. This approximation estimates the distance by identifying the anchor node that provides the minimal combined distance between $u$ and $v$. 
% \textcolor{red}{For example, in Figure~\ref{fig:distance}, we have XXX.
%The time complexity of our relative distance encoding is XXX. 
%Therefore, in comparison with existing methods~\cite{pgnn2, dist1}, our method is significantly more efficient.}

Note that our approximated shortest path distance may not be the actual shortest path distance. However,  $\hat{d}(u, v)$ actually serves as an upper bound for the true shortest path distance between $u$ and $v$. More formally, the error between the estimated distance and real distance is bounded by the parameters $c$ and $CR$:
\begin{lemma}
\label{lemma:bound1}
    Given any two nodes $u, v$ from a graph, the error of the estimated shortest path distance can be bounded by $2c$ with a probability no smaller than $1-(1-CR)^2$, where $c$ is the coverage radius and $CR$ is the coverage ratio.
\end{lemma}
\begin{proof}[PROOF]
    Given node pair $u, v$ from graph and a set of anchor nodes $\mathcal{A} = \{a_1, a_2, \ldots, a_K\}$, assume $u$ is covered by an anchor node, denoted as $a^*$, then the shortest path distance between them $d(u, a^*)\le c$. Without loss of generality, we assume $d(u, a^*)<d(a^*, v)$. Note that the following error bound still holds if $d(u, a^*)>d(a^*, v)$. The error of the estimated shortest path distance between $u, v$ is bounded by
\begin{align*}
    err(u, v)&=\hat{d}(u, v)-d(u, v)\\
    &=min_{a\in \mathcal{A}} \big( d(u, a)+d(a, v) \big)-d(u, v)\\
    &\le d(u, a^*) + d(a^*, v)-d(u, v)\\
    &\le d(u, a^*) + d(a^*, v) - |d(u, a^*) - d(a^*, v)|\\
    &=2d(u, a^*)\\
    &\le 2c
\end{align*}
The error bound holds when either $u$ or $v$ are covered by some 
anchor nodes. When neither $u$ nor $v$ is covered, this error is unbounded. The probability for this case is $(1-CR)^2$. Therefore, the probability that the error of our estimated distance is bounded is $1-(1-CR)^2$.
\end{proof}

\subsubsection{Euclidean Projection}
\label{sec:pretrain}
While anchor-based encoding enables the representation of spatial positions for nodes in a graph, it is not directly applicable for positional embeddings in LLMs. This is because shortest path distances in graph space do not correspond to distances in Euclidean space, potentially distorting actual spatial relationships. Next, we first elaborate on this argument and then present our solution. 

\begin{figure}[htbp]
\centering
\includegraphics[width=\linewidth]{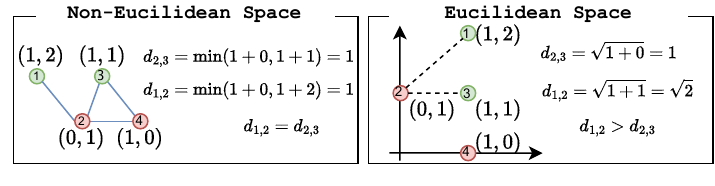}
% \vspace{-20pt}
\caption{A toy example illustrating the discrepancy between relative distance encoding in non-Euclidean graph space and the required Euclidean space for LLM positional embeddings.}
\label{fig:distance}
% \vspace{-10pt}
\end{figure}

\textbf{Mismatch between Shortest Path Distance and Euclidean Distance}.
In LLMs, positional embeddings reflect the linear order of tokens, where proximity in the sequence corresponds to closeness in the embedding space, adhering to Euclidean-like assumptions. This enables the model to capture local relationships: tokens near each other in the input sequence are also close in the learned embedding space, preserving context and meaning. However, as demonstrated in Figure~\ref{fig:distance}, when nodes 2 and 4 are set as anchor nodes, the shortest path distances between nodes 2 and 3, as well as between nodes 1 and 2, are both 1 in the graph's non-Euclidean space. In contrast, the corresponding Euclidean distances would be 1 and $\sqrt{2}$, respectively. This discrepancy in relative distances between node pairs leads to a mismatch between shortest path and Euclidean distances.

To address this issue, we propose a pretraining approach that maps the distance encoding from non-Euclidean to Euclidean space, aiming to preserve geometric relationships between nodes. The necessity of this mapping is further justified through ablation studies in Section~\ref{sec:ablation}.
The pretraining process involves a learnable function $\phi: \mathbb{R}^K \rightarrow \mathbb{R}^N$ that projects the anchor-based encoding into Euclidean space:
\begin{equation}
    \mathbf{e}_v = \phi(\hat{\mathbf{d}}_v) \in \mathbb{R}^N
\end{equation}
where $\mathbf{e}_v$ represents the transformed node embedding for node $v$. 

To preserve geometric relationships among nodes in the embedding space, we propose a rank-preserving training objective based on maximum likelihood estimation. The objective is to maximize the posterior probability $p(\Phi|>)$, where $\Phi$ denotes the parameters of the mapping function $\phi$, and $>$ represents the desired order of distances. Assuming independence for the ordering of each pair of distances, we formulate the likelihood function as:
\begin{align*}
p(>|\Phi)=&\prod_{(u,v),(i,j)\in \mathcal{E}}  p\left(\hat{d}_{\phi}(u,v)>\hat{d}_{\phi}(i,j)|\Phi\right)^{\mathbb{I}(\hat{d}(u,v)>\hat{d}(i,j))} \\
& \cdot\left(1-p\left(\hat{d}_{\phi}(u,v)>\hat{d}_{\phi}(i,j)|\Phi\right)\right)^{\mathbb{I}(\hat{d}(u,v)\leq \hat{d}(i,j))}\tag{\refstepcounter{equation}\theequation}
\end{align*}
where $\hat{d}(u,v)$ denotes the estimated distance between nodes $u$ and $v$, and $\hat{d}_{\phi}(u,v)$ represents the Euclidean distance between their corresponding mapped embeddings $\mathbf{e}_u$ and $\mathbf{e}_v$. 
We can model the probability of one distance being greater than another using the logistic function $\sigma$:
\begin{equation}
    p\left(\hat{d}_{\phi}(u,v)>\hat{d}_{\phi}(i,j)|\Phi\right):=\sigma(\hat{x}_{u,v,i,j}(\Phi)),
\end{equation}
where $\hat{x}_{u,v,i,j}(\Phi)$ denotes the difference between the Euclidean distances of the two pairs of mapped embeddings.

By maximizing the log-posterior, which is equivalent to minimizing the negative log-likelihood function, we derive the rank-preserving training objective:
\begin{align*}
\min_{\Phi} \ \ \mathcal{L} &= -\sum_{(u,v),(i,j)\in \mathcal{E}}\mathbb{I}(\hat{d}(u,v)>\hat{d}(i,j))\ln\sigma(\hat{x}_{u,v,i,j}(\Phi)) \\
 &+ \mathbb{I}(\hat{d}(u,v)\leq \hat{d}(i,j))\ln(1-\sigma(\hat{x}_{u,v,i,j}(\Phi))) \tag{\refstepcounter{equation}\theequation} 
\end{align*}
This objective function encourages the ranking of distances between nodes in the embedding space to align with the ranking of their corresponding shortest path distances in the graph. 

To facilitate practical implementation, we reformulate the objective as a binary cross-entropy (BCE) loss:
\begin{equation}
    \min_\Phi \ \ \mathcal{L} = \sum_{(u,v),(i,j)\in \mathcal{E}}\text{BCE}\left( \sigma\left( \lVert \mathbf{e}_u - \mathbf{e}_v \rVert_2 - \lVert \mathbf{e}_i - \mathbf{e}_j \rVert_2 \right), y \right),
\end{equation}
where $y$ captures the relative ordering of distances:
\begin{equation}
    y = \mathbb{I}(\hat{d}(u,v)>\hat{d}(i,j))=
    \begin{cases}
        1, & \text{if } \hat{d}(u, v) > \hat{d}(i, j), \\
        0, & \text{otherwise}.
    \end{cases}
\end{equation}

This pretraining approach ensures that the positional embeddings derived from graph structures are compatible with the Euclidean assumptions of LLM architectures while preserving the essential spatial relationships between nodes.

% Lemma~\ref{lemma:bound2} shows that an upper bound of error still exists for distance estimation by the transformed encoding.

% \begin{lemma}
%     \label{lemma:bound2}
%     [Placeholder]
% \end{lemma}
% \textbf{Complexity Analysis}.
% The time complexity of our node position encoding method mainly comes from identifying the anchor node sets and computing the shortest path distance to anchor nodes for all nodes in graph. When selecting the anchor nodes, 
% The computational complexity for pretraining is $O(N\times B), where $B$ is manually setted batch size. In consequence, the total complexity is O(N \log N + N)

\subsubsection{Time Complexity Analysis}
% \label{appendix:time}

The time complexity of the greedy algorithm for anchor node selection can be analyzed in two parts:

\paragraph{Initialization} Each node performs a BFS to construct its c-hop neighborhood, requiring $O(|\mathcal{V}|\cdot |\mathcal{E}|)$ time, where $|\mathcal{V}|$ is the number of nodes and $|\mathcal{E}|$ is the number of edges in the graph. The c-hop neighborhoods are stored for each node.

\paragraph{Anchor Selection} In each iteration, the algorithm selects an anchor and updates the coverage for remaining nodes. The worst-case time complexity for this part is $O(|\mathcal{V}|^2)$.
This is because:
    \paragraph {1}Selecting an anchor requires examining all uncovered nodes in each candidate's c-hop neighborhood ($O(|\mathcal{V}|$) in the worst case).
    \paragraph{2} After selecting an anchor, the algorithm must update the uncovered node counts for all other nodes' c-hop neighborhoods that overlap with the newly covered area ($O(|\mathcal{V}|$) nodes to update, each potentially affecting $O(|\mathcal{V}|)$ other neighborhoods).

The total time complexity is thus $O(|\mathcal{V}|\cdot|\mathcal{E}| + |\mathcal{V}|^2$.

\subsection{Task-Specific LLM Tuning}

We now focus on adapting LLMs to leverage graph-based knowledge for specific downstream tasks. Our approach integrates prompt tuning with Low-Rank Adaptation (LoRA) for efficient and effective task-specific fine-tuning.

\subsubsection{Prompt Tuning}
We employ prompt tuning to incorporate pretrained graph-based knowledge into the LLM. This technique introduces a small, trainable adapter layer that transforms our pretrained anchor-based node embeddings to soft prompts. These soft prompts serve as a learned prefix to the input, guiding the model's attention and output generation.

The generation process, including our prompt tuning adapter, can be formally expressed as:

\begin{equation}
    p_{\theta,\Phi}(Y|G,q)=\prod_{i=1}^r p(y_i|y_{<i},[\mathbf{e}_G; \mathbf{e}_T;\mathbf{e}_q]),
\end{equation}
\noindent
where $\theta$ denotes the frozen LLM parameters, $\Phi$ represents the trainable parameters of the prompt tuning adapters, $\mathbf{e}_G$ is the pretrained positional encoding derived from the graph structure, $\mathbf{e}_T$ is the textual embeddings, and $\mathbf{e}_q$ represents the question designed for corresponding graph tasks. 
The prompt tuning adapter is a shallow neural network that maps the input embeddings to a sequence of continuous prompt tokens. These tokens are prepended to the input sequence before being processed by the LLM. 
% This approach enables the LLM to adapt efficiently to our graph position encoding and downstream graph-related tasks.

\subsubsection{Low-Rank Adaptation (LoRA)}
To further enhance the LLMs' adaptability to graph-structure data, we implement Low-Rank Adaptation (LoRA)~\cite{lora} in conjunction with prompt tuning. LoRA modifies the weight update mechanism of the LLM by introducing low-rank decomposition, allowing for efficient fine-tuning of the model.
For each weight matrix $W \in \mathbb{R}^{dim_1\times dim_2}$ in the LLM, we introduce a low-rank update:

\begin{equation}
    \textbf{W}' = \textbf{W} + \textbf{BA},
\end{equation}
\noindent
where $\textbf{B} \in \mathbb{R}^{dim_1\times r}$ and $\textbf{A} \in \mathbb{R}^{r\times dim_2}$ are low-rank matrices with rank $r \ll \min(dim_1,dim_2)$. This decomposition significantly reduces the number of trainable parameters, as $r$ is typically much smaller than $dim_1$ and $dim_2$. 

During the training process, only $\textbf{A}$ and $\textbf{B}$ are updated while the original weights $\textbf{W}$ remain frozen. The update rule for the LoRA parameters can be expressed as:

\begin{equation}
    \textbf{A}_{t+1} = \textbf{A}_t - \eta \nabla_\textbf{A} \mathcal{L}(\theta, \textbf{A}_t, \textbf{B}_t),
\end{equation}
\begin{equation}
    \textbf{B}_{t+1} = \textbf{B}_t - \eta \nabla_\textbf{B} \mathcal{L}(\theta, \textbf{A}_t, \textbf{B}_t),
\end{equation}
\noindent
where $\eta$ is the learning rate, $\mathcal{L}$ is the task-specific loss function, and $t$ denotes the training iteration.

The combination of prompt tuning and LoRA in our approach enables the model to effectively incorporate graph-structural knowledge while adapting to various downstream tasks.

\begin{figure*}[htbp]
    \centering
    \includegraphics[width=0.9\linewidth]{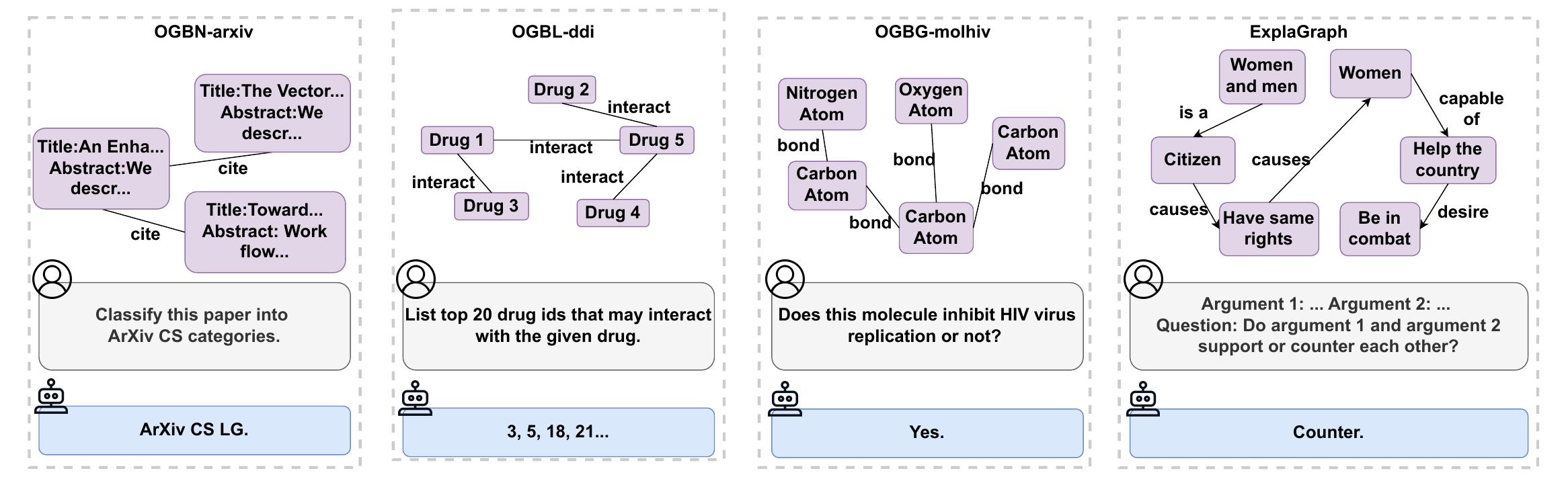}
    \vspace{-10pt}
    \caption{Illustration of dataset characteristics and LLM-based processing workflow for diverse graph-related tasks employed in our experimental setup.}
    \vspace{-10pt}
    \label{fig:dataset}
\end{figure*}

%% file: 5-Experiments.tex
\section{Experiments}
We conduct extensive experiments to demonstrate the effectiveness of our NT-LLM by investigating the following research questions:
\begin{itemize}[leftmargin=*]
    \item \textbf{RQ1}: Can NT-LLM outperform state-of-the-art methods in various graph-related tasks?
    \item \textbf{RQ2}: What does node position encoding learn? Does it capture the spatial information as intended?
    \item \textbf{RQ3}: How do different anchor selection strategies influence the performance of NT-LLM?
    \item \textbf{RQ4}: What influence do different design choices have on NT-LLM?
    \item \textbf{RQ5}: How does our tokenizer compare in efficiency to conventional message-passing GNNs and graph transformers?
\end{itemize}

\subsection{Experimental Settings}
\subsubsection{Datasets} 

We evaluate our approach on diverse graph-based tasks using benchmark datasets from Cora~\cite{cora}, the Open Graph Benchmark (OGB)~\cite{ogb}, and ExplaGraphs~\cite{expla}. Our experiments cover node classification with Cora and OGBN-arxiv, edge prediction using OGBL-ddi, and graph property prediction employing OGBG-molhiv\footnote{For OGBG-molhiv, we use the SMILES strings representing molecules as textual attributes, which are not directly provided by OGB.}. Additionally, we assess knowledge graph question answering tasks using the ExplaGraphs dataset. These datasets encompass a wide range of graph structures and task complexities, allowing for a comprehensive evaluation of our method. Table~\ref{tab:dataset} presents key statistics for each dataset, while Figure~\ref{fig:dataset} illustrates their characteristics in detail.\footnote{Cora, being a similar citation network to OGBN-arxiv, was omitted from Figure~\ref{fig:dataset} to avoid redundancy.}
\begin{table}[htbp]
    \centering
    \vspace{-5pt}
    \caption{Dataset statistics and evaluation metrics. For OGBG-molhiv and ExplaGraphs, \#Nodes and \#Edges counts represent averages across all graphs in the dataset.}
     \resizebox*{.96\linewidth}{!}{
\begin{tabular}{ccccc}
\toprule
\textbf{Dataset}&\textbf{\#Nodes}&\textbf{\#Edges}&\textbf{\#Graphs}&\textbf{Metric}\\
\midrule
Cora&2,708&10,556&1&Accuracy\\
OGBN-arxiv&169,343&1,166,243&1&Accuracy\\
OGBL-ddi&4,267&1,334,889&1&Hits@20\\
OGBG-molhiv&25.5&27.5&41,127&ROC-AUC\\
ExplaGraphs&5.17&4.25&2,766&Accuracy\\
\bottomrule
\end{tabular}}
    \label{tab:dataset}
    % \vspace{-10pt}
\end{table}

\subsubsection{Baselines}
We evaluate our proposed method against various baselines, including both traditional graph learning approaches and LLM-based methods:

\begin{itemize}[leftmargin=*]
    \item \textbf{GNN-based methods}: We incorporate widely-adopted GNN architectures, including Graph Convolutional Networks (GCN)~\cite{gcn}, Graph Attention Networks (GAT)~\cite{gat}, and GraphSAGE~\cite{graphsage}. Besides, we also evaluate two graph transformer models: GraphFormers~\cite{graphformer} and Heterformer~\cite{heterformer}.
    
    \item \textbf{LLM-only methods}: We consider approaches that process graph information directly as textual sequences using LLMs. This category includes implementations utilizing zero-shot inference, prompt tuning~\cite{ptuning}, and Low-Rank Adaptation (LoRA)~\cite{lora}.
    
    \item \textbf{GNN-LLM hybrid methods}: We compare our approach with state-of-the-art methods that integrate GNNs and LLMs. Specifically, we include GraphGPT~\cite{graphgpt} and GraphTranslator~\cite{trans}, which focus on text-attributed graph representation learning with language models. Additionally, we compare our method with G-Retriever~\cite{gretri} and GRAG~\cite{grag}, which are Graph Retrieval-Augmented Generation (RAG) methods that combine GNNs and LLMs for graph-based text generation tasks.
\end{itemize}

\begin{table*}[t]
% \tabcolsep 0.02in
    % \footnotesize
    % \small
    \centering
    \caption{Main results on benchmark datasets. The best performance is highlighted in \textbf{bold} and the second best is \underline{underlined}. $\Delta_{\text{prompt}}$ and $\Delta_{\text{LoRA}}$ represent the improvements over LLM prompt tuning and LoRA baselines, respectively.  * indicates the statistically significant improvements (\ie two-sided t-test with p<0.05) over the compared baseline.}
    % \vspace{-10pt}
     \resizebox*{0.68\textwidth}{!}{
    \begin{tabular}{cccccc}
    \toprule
    \multirow{2}{*}{\textbf{Method}} & \textbf{Cora}&\textbf{OGBN-arxiv} & \textbf{OGBL-ddi} & \textbf{OGBG-molhiv} & \textbf{ExplaGraphs} \\
    & \textbf{(Accuracy$\uparrow$)} & \textbf{(Accuracy$\uparrow$)} & \textbf{(Hits@20$\uparrow$)} & \textbf{(ROC-AUC$\uparrow$)} & \textbf{(Accuracy$\uparrow$)} \\
    % \hline
    % \multicolumn{5}{c}{GNN Baselines}\\
    \midrule
    GCN~\cite{gcn} &0.8147& 0.7360& 0.3707 & 0.7606 & - \\
    GAT~\cite{gat} &0.8352& 0.7366 & 0.4133 & 0.7520 & - \\
    GraphSAGE~\cite{graphsage} &0.8265& 0.7295 & 0.5390 & 0.7558 & - \\
    GraphFormers~\cite{graphformer}&0.8910&0.7431&0.5538&0.7414&-\\
    Heterformer~\cite{heterformer}&0.8761&0.7390&0.5482&0.7505&-\\
    % \cdashline{1-6}
    \midrule
    % \multicolumn{5}{c}{LLM Baselines}\\
    % \hline
    zero-shot&0.6490&0.5406&0.3384&0.6321&0.6679\\
    prompt tuning~\cite{ptuning}&0.7903&0.6971&0.3592&0.6554&0.8224\\
    LoRA~\cite{lora}&0.8194&0.7323&0.3918&0.7529&0.9296\\
    % \cdashline{1-6}
    % \multicolumn{5}{c}{LLM-GNN Baselines}\\
    \midrule
    GraphGPT~\cite{graphgpt} &0.9085&0.7637&0.5011&\underline{0.7851}&0.9052\\
    GraphTranslator~\cite{trans}&0.9351&\underline{0.7748}&{0.5425}&0.7764&0.9273\\
    G-Retriever~\cite{gretri}&0.9148&0.7521&0.4573&0.6920&0.9231\\
    G-Retriever LoRA&0.9350&0.7580&0.5296&0.7635&0.9240\\
    GRAG~\cite{grag}&0.9296&0.7492&0.4617&0.6698&0.9242\\
    GRAG LoRA&0.9473&0.7554&0.5386&0.7309&\underline{0.9422}\\
    % \hline
    NT-LLM &\underline{0.9478}&0.7525 &\underline{0.5904} &0.7531 & 0.9332 \\
    $\Delta_{\text{prompt}}$&$\uparrow 19.93\% ^*$&$\uparrow 7.95\%^*$&$\uparrow 74.47\%^*$&$\uparrow 14.91\%^*$&$\uparrow 13.47\%^*$\\
    NT-LLM LoRA &\textbf{0.9531}&\textbf{0.7752} & \textbf{0.6375}& \textbf{0.8045}&\textbf{0.9603}\\
    $\Delta_{\text{LoRA}}$&$\uparrow 16.32\%^*$&$\uparrow 3.02\%^*$&$\uparrow 62.71\%^*$&$\uparrow 6.85\%^*$&$\uparrow 3.30\%^*$\\
    \bottomrule
    \end{tabular}
    }
    \label{tab:main_results}
    % \vspace{-10pt}
\end{table*}

\subsection{Implementation Details}
We implement all models and experiments using PyTorch~\cite{pytorch}, PyTorch Geometric~\cite{pyg}, and the HuggingFace Transformers~\cite{libtrans} libraries. All experiments are conducted on two NVIDIA RTX 6000 Ada GPUs, each with 48GB memory.

\subsubsection{Text and LLM Components.} 
\label{llm}
For encoding textual attributes, we employ SentenceBERT~\cite{sbert}. The LLM component of all experiments is based on the pretrained LLaMA3-8B~\cite{llama}. We use LLaMA3-8B in zero-shot (no fine-tuning), as well as in prompt-tuning and LoRA-based fine-tuning settings. During LLM fine-tuning with LoRA, we set the low-rank dimension to 8 and the scaling factor to 16. Optimization uses AdamW~\cite{adamw} with a learning rate of 1e-4 and weight decay of 0.05. Fine-tuning runs for a maximum of 10 epochs with an early stopping patience of 3. The batch size is set to 32 for OGBN-arxiv and OGBL-ddi, and to 2 for OGBG-molhiv and ExplaGraphs, according to dataset size.

\subsubsection{GNN-based Methods.}
Our baseline and hybrid GNN models use a 4-layer architecture with hidden dimensions of 256, ReLU activation, and a dropout rate of 0.5. 
Graph transformer baselines utilize nested GAT architecture combined with transformer layers, where each node uses 5 uniformly sampled neighbors as context. Training runs using the AdamW optimizer for 500 epochs with an early stopping patience of 10, learning rate of 1e-3 and weight decay of 5e-4.

\subsubsection{NT-LLM Implementation.}
In the node tokenizing stage, we set the anchor identification parameters as $c=1$ and $CR=0.7$, and map node encodings via a 3-layer MLP. In the LLM fine-tuning stage, we following the settings in~\ref{llm}.

\subsubsection{GNN-LLM Hybrid Baselines.}
For GNN-LLM hybrid methods, we combine a 4-layer GAT with LLaMA3-8B, following the architecture and hyperparameter settings as described in their papers.

\subsection{Main Results (RQ1)}

Table~\ref{tab:main_results} compares the performance of our proposed NT-LLM method against baselines on five benchmark datasets on the corresponding task, respectively.\footnote{GNNs are unable to perform complex graph reasoning tasks in the ExplaGraphs dataset, thus the corresponding cells are marked with -.}
% NT-LLM consistently outperforms all baselines across various tasks and datasets, demonstrating its effectiveness and broad applicability. 
We have the following key findings:

\begin{itemize}[leftmargin=*]
    \item \textbf{NT-LLM consistently outperforms all baseline methods across various tasks and datasets. } This observation justifies the superiority of NT-LLM and demonstrates its effectiveness and broad applicability in graph learning. 
    % \item \textbf{NT-LLM outperforms LLM baselines on all datasets}, showing the advantages of integrating graph structure into LLMs for graph-related tasks. Its superior performance over prompt tuning and LoRA methods indicates NT-LLM's ability to effectively capture both textual and topological information.
    \item \textbf{NT-LLM effectively addresses the challenge of enabling LLMs to understand graph structures}. In other words, NT-LLM leverages the strengths of LLMs in understanding textual attributes while benefiting from our proposed node position encoding to capture the graph topology.
    First, NT-LLM outperforms pure LLM and GNN baselines on all datasets. This observation demonstrates that understanding textual attributes and topology are equally important for graph learning tasks. 
    Second, when fine-tuning NT-LLM with LoRA (fine-tuned NT-LLM), its performance surpasses LLM-GNN hybrid approaches. 
    This suggests that NT-LLM is more effective at enabling LLMs to understand graph structures compared to intermediate solutions, \ie LLM-GNN hybrid approaches.
    % This suggests that our method, when combined with LoRA techniques, strikes an optimal balance between utilizing graph topology and leveraging LLM knowledge.
    \item \textbf{The superiority of NT-LLM in graph understanding comes from our proposed node position encoding}. In particular, the OGBL-ddi dataset lacks textual attributes. As we can see, LLM methods perform worse than GNN baseline methods, which highlights their limitations in capturing topological information from graph data. Unlike LLM methods, our proposed NT-LLM, despite not using GNNs, outperforms all baselines with over 60\% improvement compared to LLM methods, demonstrating its ability to effectively encode graph structure.
\end{itemize}

In conclusion, NT-LLM shows superior performance and adaptability across various graph-related tasks and datasets. The improvements over state-of-the-art baselines, even in the absence of textual attributes, highlight the effectiveness of our proposed method in capturing both textual and structural information.

% \vspace{-10pt}
\subsection{Understanding Node Position Encoding (RQ2)}
To understand what node position encoding learns, in this section, we provide visualization for the learned node position embedding on the Cora dataset to gain further insights. 
We select this dataset because, in Cora, nodes from the same class tend to be naturally closer in the graph structure. This property allows us to directly evaluate the quality of the node position embeddings by observing how well they align with the class labels.

Figure~\ref{fig:visual_on_cora} illustrates the embeddings before and after the transformation in positional embedding pretraining, shown against class labels. 
Prior to the transformation, nodes belonging to the same class can be separated distantly in the embedding space. 
However, after applying the transformation, these nodes are effectively projected into the same region, highlighting the efficacy of our pretraining approach in capturing the underlying semantic relationships among nodes.
For instance, the green dots, which are dispersed before the transformation, become densely clustered afterward.

% \begin{figure}[htbp]
%     \centering
%     \includegraphics[width=\linewidth]{figs/cora.png}
%     \caption{Embeddings visualized before (left) and after (right) the transformation in pretraining. The colors represent the ground truth labels of nodes.}
%     \label{fig:visual_on_cora}
% \end{figure}

\begin{figure}[t]
    \centering
    \vspace{-5pt}
    \includegraphics[width=\linewidth]{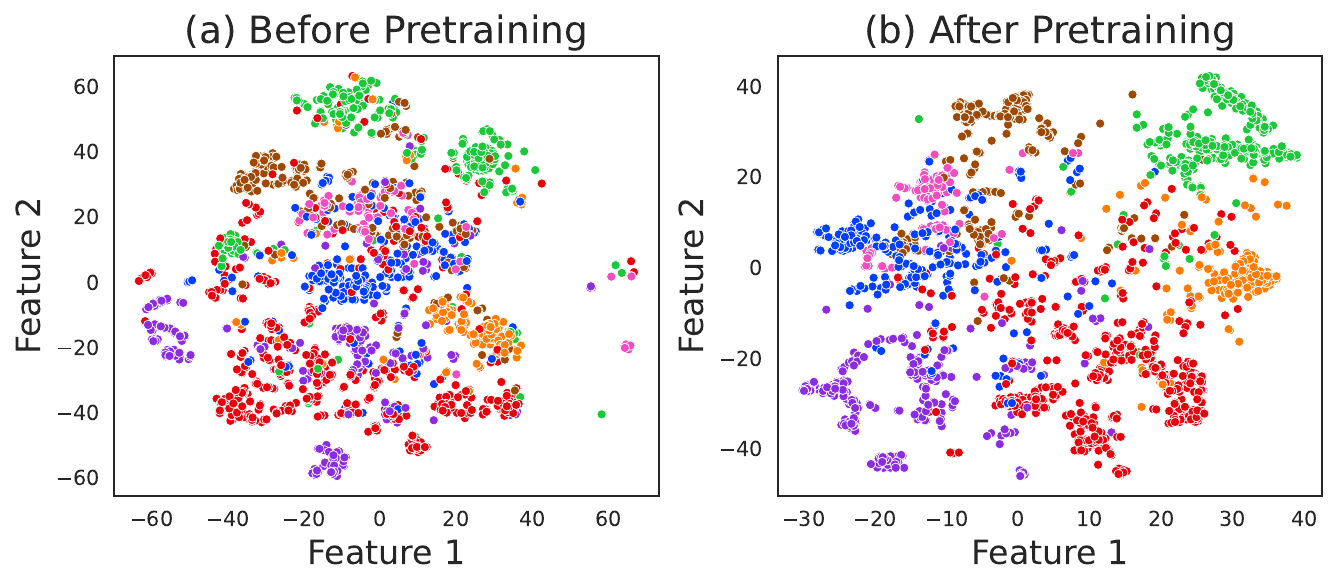}
    % \vspace{-10pt}
    \caption{Embeddings visualized before and after the transformation in pretraining on Cora dataset. The colors represent the ground truth labels of nodes.}
    \label{fig:visual_on_cora}
    \vspace{-15pt}
\end{figure}

% \vspace{-5pt}
\subsection{Anchor Selection Strategies Impact (RQ3)}
Since anchor nodes offer a comprehensive view of the graph structure, different strategies for identifying anchor nodes may impact NT-LLM’s ability to comprehend the graph. 
In this section, we conduct an extensive evaluation of various anchor selection strategies, on three datasets, \ie Cora, OGBN-arxiv and OGBL-ddi, using a fixed seed and the NT-LLM architecture.
Subsequently, the positional embeddings are pretrained following the same procedure outlined in Section~\ref{sec:pretrain}.
% The inputs are positional embeddings derived from different anchor sets, all subjected to identical pretraining procedures. 
\begin{table}[htbp]
\centering
\small
\caption{Comparison of anchor selection strategies across three datasets. The highest performance for each dataset is shown in \textbf{bold}.}
\vspace{-5pt}
\label{tab:selection}
 \resizebox*{.96\linewidth}{!}{
\begin{tabular}{cccc}
\toprule
\textbf{Strategy} & \textbf{Cora}& \textbf{OGBN-arxiv} & \textbf{OGBL-ddi} \\
\midrule
Degree & 0.9172 & 0.7312 & 0.5731 \\
Random & 0.8891 & 0.6783 & 0.5019 \\
Closeness~\cite{closeness} & 0.8931 & 0.6392 & 0.4852 \\
Eigenvector~\cite{lap2} & 0.8424 & 0.6105 & 0.4736 \\
PageRank~\cite{pagerank} & 0.8703 & 0.6641 & 0.5127 \\
Betweenness~\cite{betweenness} & 0.8539 & 0.6428 & 0.4967 \\
HPLC~\cite{dist2}&0.9174&0.7411&0.5613\\
% \hline
Ours & \textbf{0.9478} & \textbf{0.7525} & \textbf{0.5904} \\
\bottomrule
\end{tabular}}
\end{table}

Table~\ref{tab:selection} presents the experimental results. Our method achieves the best performance among all evaluated strategies, surpassing traditional centrality-based approaches (such as Degree and PageRank~\cite{pagerank}), random selection, and the landmark-based HPLC~\cite{dist2}.
% %
% \begin{table}[htbp]
% \centering
% % \small
% \caption{Comparison of anchor selection strategies across three datasets. The highest performance for each dataset is shown in \textbf{bold}.}
% \label{tab:selection}
% \begin{tabular}{c|ccc}
% \hline
% \textbf{Strategy} & \textbf{Cora}& \textbf{OGBN-arxiv} & \textbf{OGBL-ddi} \\
% \hline
% Degree & 0.9172 & 0.7312 & 0.5731 \\
% Random & 0.8891 & 0.6783 & 0.5019 \\
% Closeness & 0.8931 & 0.6392 & 0.4852 \\
% Eigenvector & 0.8424 & 0.6105 & 0.4736 \\
% PageRank & 0.8703 & 0.6641 & 0.5127 \\
% Betweenness & 0.8539 & 0.6428 & 0.4967 \\
% Matrix Fraction & 0.7153 & 0.5846 & 0.4329 \\
% HPLC&0.9174&0.7411&0.5613\\
% \hline
% Ours & \textbf{0.9478} & \textbf{0.7525} & \textbf{0.5904} \\
% \hline
% \end{tabular}
% \end{table}
% %

To provide a clearer insight into the advantage of our anchor selection strategy, we compare the anchor nodes selected by different strategies on Cora dataset in Figure~\ref{fig:dis}. 
% To gain a clearer insight into the distribution of anchor nodes within the graph, we provide a visualization on Cora dataset in Figure~\ref{fig:dis}. 
The anchor nodes selected by our method are more evenly distributed across the graph structure. 
In contrast, methods such as Degree, HPLC, Closeness, PageRank, and Eigenvector focus on selecting “important” nodes but fail to provide broad coverage, particularly of nodes located at considerable distances from the graph’s central area.

% successfully identifying anchors in peripheral regions that are typically challenging for alternative algorithms to achieve. 
% This broader coverage includes nodes situated at considerable distances from the graph's central area.

\begin{figure}[htbp]
    \centering
    \includegraphics[width=\linewidth]{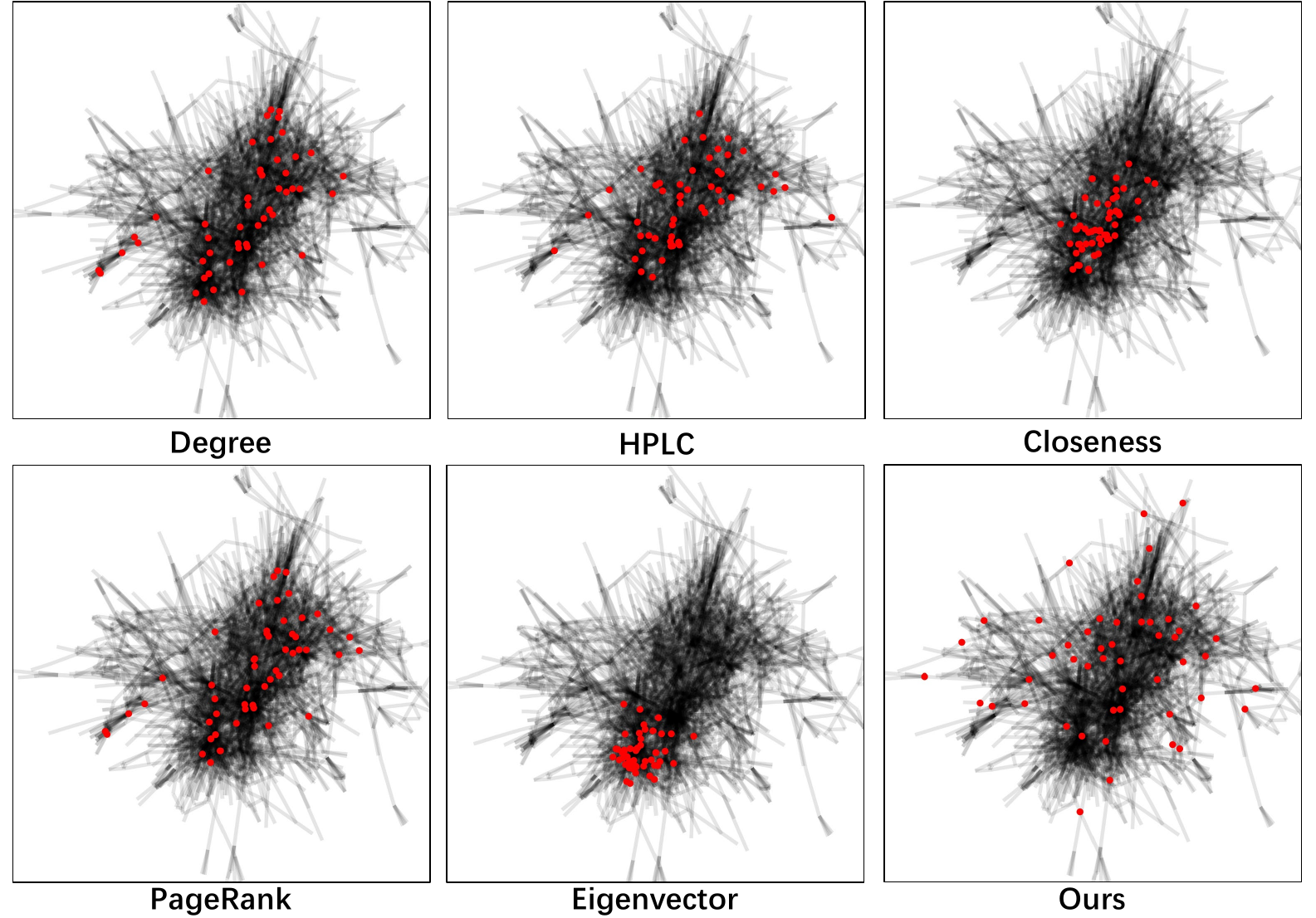}
    % \vspace{-10pt}
    \caption{Distribution of anchor nodes (marked in red) selected by different strategies on the Cora dataset. Our method achieves a more even distribution, effectively covering the peripheral regions of the graph.}
    \label{fig:dis}
    \vspace{-5pt}
\end{figure}

\vspace{-5pt}
\subsection{Ablation Studies (RQ4)}
\label{sec:ablation}
In this section, we conduct extensive ablation studies to investigate the effectiveness of each component in NT-LLM, and justify our model design choices. 

% To further investigate the contribution of each component in our proposed method and validate the effectiveness of our design choices, we conducted comprehensive ablation studies on multiple datasets.

\subsubsection{Impact of Model Components}
NT-LLM has specific design features, including the node position encoding, its corresponding pretraining task, and two different strategies for LLMs to leverage node position encoding, \ie prompt tuning and low-rank adaptation. 
We evaluate the performance of each variant of our model on five datasets as follows:

\begin{itemize}[leftmargin=*]
    % \item \textbf{Full}: The complete NT-LLM model.
    \item \textbf{w/o PE}: The NT-LLM without positional encoding, using raw node features as input to the LLM. 
    \item \textbf{w/o Pre}: The NT-LLM without the distance transformation pretraining module, using concrete anchor-based distances as node position embeddings.
    \item \textbf{w/o PT}: The NT-LLM without the prompt tuning module, directly inputting all embeddings into the LLM.
\end{itemize} 

Table~\ref{tab:ablation} presents the results of the ablation study, which evaluates the impact of removing individual components from the proposed method. The observed performance drop across all datasets confirms the importance and complementary nature of each component within the method.
In particular, we observe that node position encoding pretraining is critical for NT-LLM. The variant without pretraining (w/o Pre) experiences a significant performance drop when the pretraining module is removed, supporting our argument in Section~\ref{sec:pretrain}. This is due to the mismatch between shortest path and Euclidean distances, which distorts actual spatial relationships. Therefore, positional embedding pretraining is an indispensable component of NT-LLM.

% The full NT-LLM model consistently achieves the best performance, highlighting the importance of each component in capturing graph structure and adapting the LLM to the task at hand.

\begin{table}[t]
\tabcolsep 0.02in
    \centering
    % \vspace{-5pt}
    \caption{Performance comparison of NT-LLM variants across four datasets. Best results for each dataset are in \textbf{bold}.}
    % \vspace{-5pt}
     % \resizebox*{.96\linewidth}{!}{
    \begin{tabular}{cccccc}
    \toprule
    \textbf{Variant} & \textbf{Cora} &\textbf{arxiv} & \textbf{ddi} & \textbf{molhiv} & \textbf{ExplaGraphs} \\
    \midrule
    NT-LLM &\textbf{0.9478}&\textbf{0.7525}&\textbf{0.5904}&\textbf{0.7531}& \textbf{0.9332} \\
    w/o PE &0.8070&0.6971&0.3592&0.6554& 0.8224 \\
    w/o Pre &0.8195&0.6538&0.3791&0.6419& 0.7671 \\
    w/o PT &0.7864&0.5904&0.3460&0.5834&  0.7024 \\
    \bottomrule
    \end{tabular}
    % }
    \label{tab:ablation}
    \vspace{-10pt}
\end{table}

% \subsubsection{Effectiveness of Anchor Selection Strategies}
% Since anchor nodes offer a comprehensive view of the graph structure, different strategies for identifying anchor nodes may impact NT-LLM’s ability to comprehend the graph. In this section, we conduct an extensive evaluation of various anchor selection strategies, using a fixed seed and the NT-LLM architecture.
% Experiments are conducted on three datasets: Cora, OGBN-arxiv and OGBL-ddi. The inputs are positional embeddings derived from different anchor sets, all subjected to identical pretraining procedures. Table~\ref{tab:selection} presents the experimental results. Our method achieves best performance among all evaluated strategies, outperforming both traditional centrality-based approaches (such as Degree and PageRank~\cite{pagerank}), random selection and landmark-based HPLC~\cite{dist2}.

\subsubsection{Impact of Hyperparameters}

\begin{figure}[h]
    \centering    \subfloat{\includegraphics[width=0.45\textwidth]{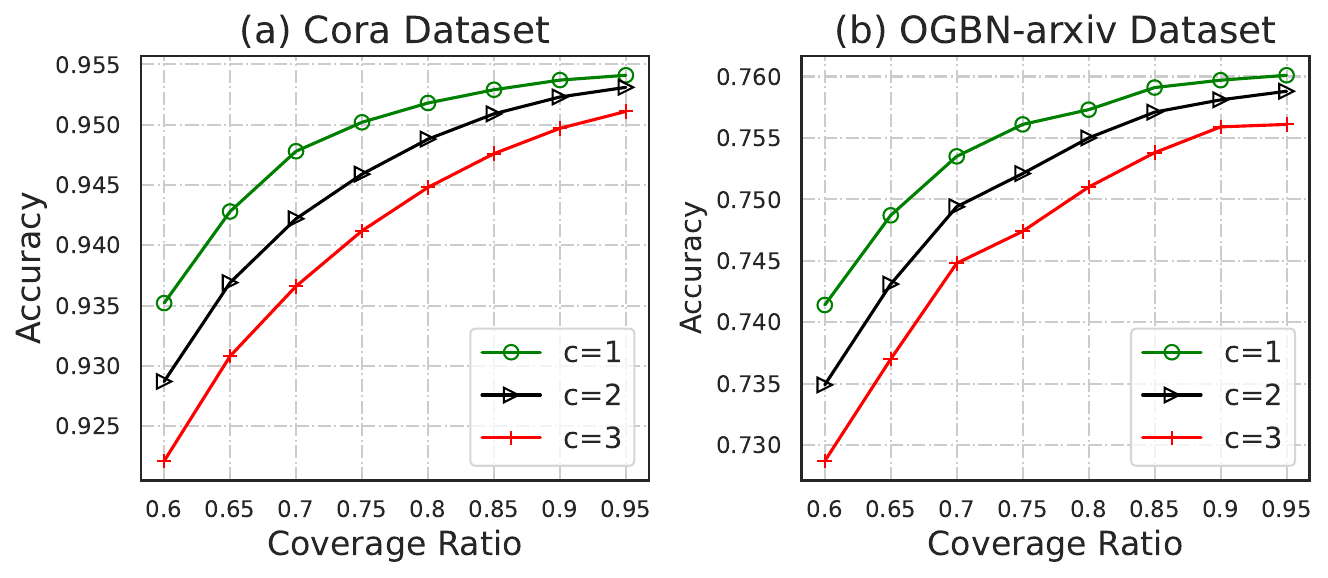}}
    \hfill
    \subfloat{\includegraphics[width=0.45\textwidth]{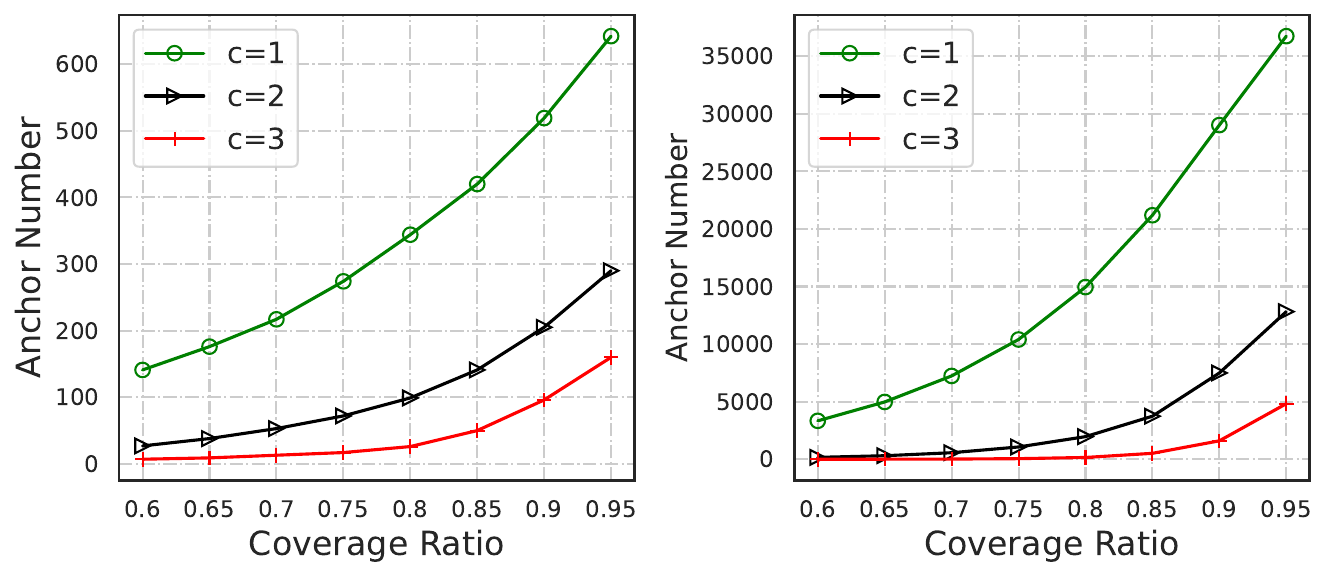}}
    % \vspace{-10pt}
      \caption{Effects of coverage radius ($c$) and coverage ratio ($CR$) on model accuracy and the number of anchor nodes for the Cora and OGBN-arxiv datasets. The top row shows the impact on model accuracy, while the bottom row illustrates the changes in the number of anchor nodes as $c$ and $CR$ vary.}
    \label{fig:hyper}
    \vspace{-10pt}
\end{figure}
We investigate the impact of two key hyperparameters in NT-LLM: the coverage radius $c$ and the coverage ratio $CR$. Figure \ref{fig:hyper} presents the relationships between these hyperparameters, model accuracy and the number of anchor nodes. 
The results demonstrate that smaller values of $c$ and larger values of $CR$ generally lead to a better performance. This trend aligns with the error bound established in Lemma~\ref{lemma:bound1}. Notably, we observed that the number of anchor nodes increases exponentially as $c$ decreases and $CR$ increases. This relationship underscores the importance of carefully selecting these hyperparameters to balance computational complexity and model performance.

\subsection{Tokenizer Efficiency (RQ5)}
The only trainable component in our proposed graph tokenizer is a simple MLP, making it intuitively much more efficient than conventional message-passing GNNs or graph transformers. To validate this, we compare the efficiency of various graph tokenizers, including our own, across multiple datasets. The results are summarized in Table~\ref{tab:eff}. The consistently lower number of trainable parameters and training time demonstrate the efficiency of our tokenizer.

\begin{table}[h]
    \centering
    \caption{Comparison of different tokenization methods based on the number of trainable parameters and training time.}
    \label{tab:eff}
     \resizebox*{.96\linewidth}{!}{
    \begin{tabular}{cccc}
    \toprule
    Dataset & Tokenizer & Trainable Parameters & Training Time\\
    \midrule
   \multirow{3}{*}{Cora}  & GAT & 1.4M & 2min\\
         & GraphFormer & 3.6M & 10min\\
         & NT-LLM & 0.3M & <1min\\
         \midrule
        \multirow{3}{*}{OGBN-arxiv} & GAT & 21.7M & 4h\\
        & GraphFormers & 49.2M & 6h\\
        & NT-LLM & 0.7M &10min\\
        \midrule
        \multirow{3}{*}{OGBL-ddi}& GAT & 2.6M & 3min \\
        &GraphFormers& 6.1M&13min\\
        &NT-LLM & 0.5M & 1min \\
    \bottomrule
    \end{tabular}
}
\end{table}